\documentclass[11pt,letterpaper]{article}
\usepackage[margin=1in]{geometry}
\usepackage{natbib}

\usepackage[utf8]{inputenc} 
\usepackage[T1]{fontenc}    
\usepackage{amsmath,amssymb,amsthm,mathtools}
\usepackage{graphicx,booktabs,multirow,array,longtable, subcaption}
\usepackage{algorithm,algpseudocode}
\usepackage{microtype}
\usepackage[hidelinks]{hyperref}
\usepackage{xurl}

\hypersetup{
    colorlinks=true,
    linkcolor=blue,
    citecolor=blue,
    urlcolor=blue,
    pdfauthor={},
    pdfsubject={},
    pdfkeywords={}
}

\newcommand{\R}{\mathbb R}
\newcommand{\E}{\mathbb E}
\newcommand{\KL}{\operatorname{KL}}
\newcommand{\Poi}{\operatorname{Poisson}}
\newcommand{\Unif}{\operatorname{Uniform}}
\newcommand{\diag}{\operatorname{diag}}
\newcommand{\sg}{\operatorname{sg}}
\newcommand{\sigmoid}{\operatorname{sigmoid}}
\newcommand{\logit}{\operatorname{logit}}
\newcommand{\ind}{\mathbf 1}
\newtheorem{proposition}{Proposition}
\newtheorem{theorem}{Theorem}
\newtheorem{assumption}{Assumption}

\title{Scan-Parallel Variational Inference for event-driven\\ Latent Jump Models}

\author{Vansh Bansal\thanks{Corresponding author: Vansh Bansal (\texttt{vansh@utexas.edu})} , Min Chen, \& James G.~Scott \\
Department of Statistics and Data Science\\
UT Austin
}

\begin{document}
\date{}

\maketitle

\begin{abstract}
Continuous-time latent state-space models are widely used for irregularly
observed time series, typically representing uncertainty through a random initial state followed by smooth ordinary differential equation (ODE) dynamics, or  Brownian perturbations accumulated by a stochastic differential equation (SDE). While these constructions are natural for smooth or diffusive systems, they do not directly encode a finite set of unobserved shocks whose timing, magnitude, and effect on memory are themselves uncertain. We instead introduce latent jump models (LJMs): the latent state follows closed-form deterministic dynamics between Poisson-driven event times and undergoes a learned instantaneous jump at each event, with state-dependent intensity and mark distributions. This defines a time-homogeneous Markov process whose jumps are intrinsic state transitions rather than numerical integration steps. For training, we derive a tractable event-augmented path-space evidence lower bound (ELBO) and introduce \textit{Scan-Parallel Variational Inference (SPVI)}: an observation-conditioned variational posterior samples the complete latent event configuration before the state recurrence is evaluated using triangular affine scans. This enables fast and scalable training for LJMs in \textit{logarithmic temporal depth}. Across synthetic stochastic systems and real irregular time series, LJMs remain competitive with continuous-time latent baselines in forecast accuracy, often with fewer latent updates than numerical solver evaluations, while training one to two orders of magnitude faster than the sequential baselines.
\end{abstract}

\section{Introduction}
Continuous-time latent state-space models provide a natural language for irregularly observed time series: a latent state evolves between observation times, and an emission model relates that state to the measurements. Latent ODE models are especially appropriate when the dynamics are smooth and nearly deterministic conditional on the initial state \citep{chen2018,rubanova2019}. Latent SDE models add Brownian motion, allowing uncertainty to accumulate continuously and producing rough but continuous latent paths \citep{li2020}.
These choices give expressive models of smooth and diffusive evolution but many systems instead evolve through relatively regular motion interrupted by a finite number of abrupt, unobserved changes: interventions, failures, spikes, flares, price shocks, or transitions between operating conditions. In the latter regime, it may be scientifically and computationally preferable to model their times and marks as latent random variables rather than forcing them to appear as a steep continuous vector field or as the accumulation of many Gaussian perturbations.

\begin{figure}[t]
    \centering
    \includegraphics[width=\linewidth]{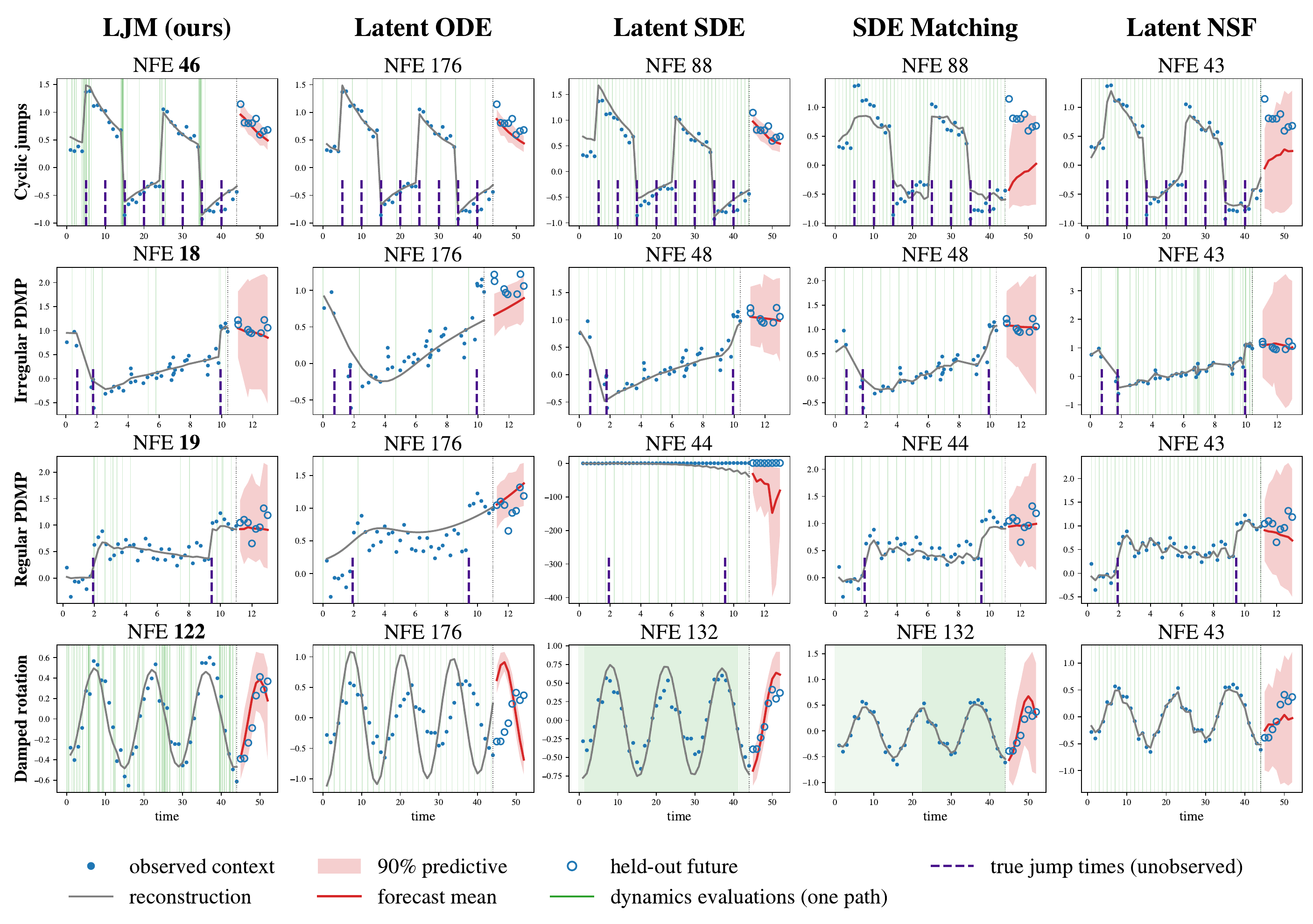}
    \caption{Comparison across four event-driven and smooth systems. All models use identical history-encoder and emission architectures. Solver step sizes are coarsened for baselines at evaluation to approach LJM event budget; NFE labels show the number of dynamics evaluations. LJM shows better reconstruction and forecast quality at fewer NFE. Latent SDE diverges on the regular PDMP, while Latent ODE and Neural Stochastic Flow (solver-free) exhibit errors in several examples.}
    \label{fig:illustration}
\end{figure}

We take this event-driven perspective and develop \emph{latent jump models} (LJMs), representing stochasticity through a sequence of latent events governed by a marked point process with Poisson noise. A continuous-time latent state evolves through closed-form decay between random event times and undergoes a learned marked-jump at each event. The observations remain (irregular) time-series measurements; however, the events themselves are latent. Therefore, events act as stochastic drivers of the latent state, whose intensity determines when the state is likely to jump and a random mark determines its character. Both the prior event intensity and mark distribution depend on the current state, so the consequences of one event may change the distribution of the later ones.  

This produces a latent piecewise-deterministic Markov process (PDMP) \citep{davis1984}, admitting a Chapman–Kolmogorov semigroup. Its trajectories can be simulated exactly, without numerical ODE or SDE solvers, by combining the closed-form flow with jumps induced by the latent event process. This differs from neural point process models in which events themselves are typically the observed data. Therefore, learned events in LJMs may differ from the (typically unidentifiable) events of the data-generating process, however still provide adaptive stochastic updates to the latent representation which are learned rather than enforced by numerical solvers; see Figure \ref{fig:illustration}. Moreover, as we show in Section \ref{sec:elbo}, replacing diffusion with jumps does not even require giving up variational learning on the path space. By constructing the variational posterior as a \emph{controlled marked Poisson process}, we derive a path-space evidence lower bound (ELBO) using a change-of-measure formula for marked-point processes \citep{bremaud1981}. The posterior shares its deterministic dynamics with prior process, but acts as a control policy on the latent state, adjusting the initial state law, intensity at which events occur and shaping the distribution of event marks in response to observed data. 

However, a tractable event likelihood is not enough. Event-driven simulation alone does not remove sequential state unrolling. The situation mirrors Latent ODEs and SDEs, whose training integrates the posterior path with a numerical solver: every step reads the state produced by the previous one, so the depth of a forward–backward pass grows with the number of solver steps, and stiff or long windows are paid for in sequential time \citep{chen2018, li2020}. A latent jump model inherits the same structure through its events. With a random number $L$ of events in a window, a general neural jump map $F_\theta(z,u)$ has to be applied $L$ times in order, because the state entering event $n$ is the state produced by event $n-1$. Sampling every event time and mark in advance does not remove this: the recurrence is in the state, not in the events; this serial depth makes training slow. Recent work removes simulation from continuous-time training in different ways including 
matching posterior marginals via SDE Matching \citep{bartosh2025} and by learning stochastic transition laws for solver-free sampling via Neural Stochastic Flows \citep{kiyohara2025}. We take a different route: we keep an explicit finite event configuration and a standard path-space objective, and change the transition so that the recurrence itself can be evaluated in parallel.

Hence, we introduce \textit{Scan-Parallel Variational Inference} (SPVI) that enables fast and scalable path-space training of event-driven state space models. Our central construction parallelizes this chain so that its depth scales \textit{logarithmically} in the number of events $L$ without approximating the dynamics. Two design decisions make it possible. First, the variational event process is exogenous: its initial state law, intensities and mark laws are functions of the observations through an encoder network, not of the sampled latent state. Therefore the complete event configuration, i.e., event times and marks, can be drawn before any state is computed. Second, the jump map is constructed to be block-triangular and affine within a block. The state is split into $J$ ordered blocks. Write $z^{(j)}_{n}$ for block $j$ just after event $n$ and $\tilde z^{(j)}_{n}$ for the same block just before it, i.e., the value after event $n-1$ carried forward by the closed-form decay. Event $n$ then updates block $j$ as
$$z^{(j)}_{n} = R^{(j)} \odot \tilde z^{(j)}_{n}+ B^{(j)},$$ where the retention gates $R^{(j)}\in(0,1)$ and the shifts $B^{(j)}$ are produced by a neural network from three inputs: mark of event $n$, lower blocks $\tilde z^{(<j)}_n$ just before, and lower blocks $z^{(<j)}_n$ just after the event $n$. The network never reads block $j$'s own value $\tilde z^{(j)}_{n}$, and it reads nothing from any later event.

Together, these decisions make every step of block $j$ an element-wise affine map whose coefficients are known before the state is. The marks are known because events are sampled first and the lower blocks because they are computed first. We emphasize that nonlinearity is relocated rather than removed: the coefficient networks are arbitrary functions of the marks and of the lower blocks, so the transition is nonlinear across blocks. However, the closure property of affine maps enables a parallel prefix scan \citep{kogge1973parallel, ladner1980parallel} which yields the state of block $j$ after every event in $O(\log(1+L))$ dependent stages. A recurrence of length $L$ therefore becomes $J$ parallel scans yielding $O(J\log(1+L))$ depth. SPVI therefore reproduces the chronological trajectory exactly while changing the order of computation from a length-$L$ event recurrence to \textit{$J$ parallel scans, each with logarithmic temporal depth}, enabling fast and scalable training for long time series.

We evaluate our method on six synthetic continuous-time systems spanning smooth rotations, cyclic jumps and piecewise-deterministic processes, and on two real benchmarks, REFIT \citep{murray2017electrical} electrical load measurement and USHCN \citep{menne2006ushcn} climate datasets. On event-driven systems, LJM matches the best forecast error with a fraction of the baselines' dynamics evaluations while remaining competitive even on smooth rotations. On both real datasets, LJM remains competitive with baselines in terms of forecasting error, at one to two orders of magnitude less training time.

\textbf{Summary of contributions.}
We introduce a finite-activity latent jump model with state-dependent
event laws, derive its event-augmented variational objective, and develop
a triangular transition architecture that enables scan-parallel posterior
computation. Experiments characterize the predictive role of latent
events and demonstrate shorter training times than the evaluated
baselines, with matched-model benchmarks isolating the benefit of
scan-based execution.

\section{Related work}
\paragraph{Continuous-time latent models.}
Neural ODE \citep{chen2018} learns vector fields, and Latent ODE \citep{rubanova2019} extends them to probabilistic modeling of irregular observations. Latent SDE \citep{li2020} introduces continuously accumulating process noise and support variational learning through changes of drift. NCDSSM \citep{ansari2023} combines neural observation models with continuous-discrete filtering and smoothing; continuous recurrent units use linear latent SDEs with closed-form recursive computations \citep{schirmer2022}.

\paragraph{Simulation-free and solver-free alternatives.}
Recent methods reduce numerical integration by changing the representation of continuous-time dynamics or their variational posterior. Neural Flows \citep{bilos2021} directly parameterize invertible ODE solution maps, with latent trajectories determined by their initial state; the autonomous semigroup identity is an optional regularization target rather than an architectural constraint. SDE Matching \citep{bartosh2025} instead retains an SDE prior and a path-space ELBO, parameterizes the posterior through one-time marginals, and constructs a compatible drift so that the objective can be estimated from states sampled at random times without integrating posterior trajectories. Amortized reparametrization uses a related marginal-based construction for latent SDE inference \citep{course2023}. However, \cite{smith2026helmholtz} show that, for diffusions with a common diffusion coefficient, one-time marginals do not uniquely determine the path law: compatible drifts can differ by a field $v_t$ satisfying $\nabla\!\cdot(q_t v_t)=0$. Selecting a particular compatible drift can therefore restrict posterior temporal dependence; Helmholtz-SDE~\citep{smith2026helmholtz} addresses this restriction by optimizing over marginal-preserving drift corrections. Neural Stochastic Flows  \citep{kiyohara2025} take a different approach, learning transition kernels for arbitrary time gaps with conditional normalizing flows and regularizing one-step and composed kernels toward Chapman--Kolmogorov consistency. This transition-kernel construction is distinct from the one-time-marginal parameterizations analyzed in Helmholtz-SDE. 

SPVI instead specifies a joint variational law over the initial state and \textit{complete latent event configuration}. The shared flow and jump map induce the state path, and the ELBO compares the augmented event laws. Temporal dependence thus follows from the event construction, while associative scans replace chronological posterior state unrolling. Its structural restriction is different: the variational event law depends on observations but not on the sampled latent state.

\paragraph{Jump processes and parallel state computation.}
PDMPs provide the classical framework for deterministic motion interrupted by random jumps \citep{davis1984}. Recurrent switching linear dynamical systems make regime transitions depend on the continuous state \citep{linderman2017}; LJM instead applies marked transformations to the state while retaining the same between-event flow. Variational inference for Markov jump processes has been developed through mean-field \citep{opper2007} and moment-based approximations \citep{wildner2019}. LJM uses a continuous state and an observation-conditioned marked-event posterior. Neural Jump SDEs learn hybrid dynamics for observed marked events \citep{jia2019}, whereas Neural Jump ODEs \citep{herrera2021} and GRU-ODE-Bayes \citep{debrouwer2019} update their predictive states when observations arrive. TriTPP \citep{shchur2020fast} enables parallel sampling and density evaluation of event times; its variational application to latent Markov jump processes combines inferred jump times with exact conditional inference over finite states. NeuralMJP \citep{neuralmjp2023} learns observation-conditioned transition rates for a discrete-state posterior and evolves its probabilities through a master equation.

LJM instead places a variational distribution over unobserved event times and marks. Its per-block recurrence is an event-indexed analogue of diagonal recurrences in S4D \citep{gu2022s4d} and LRU \citep{orvieto2023}, with event-dependent coefficients resembling Mamba's input-dependent transitions \citep{gu2024mamba}. Associative formulations also support Bayesian smoothers \citep{sarkka2021} and state-space layers \citep{smith2023}. S2P2 \citep{chang2025} models observed marked events using stacked state-space layers with additive event impulses. LJM, however, applies scans to a sampled latent event stream and trains through an event-augmented ELBO. Its triangular dependence resembles affine coupling in RealNVPs \citep{dinh2017}, but upper-block coefficients depend nonlinearly on realized lower-block states, requiring successive scans rather than one globally linear recurrence.

\section{Latent Jump Models}

In this section, we formally describe our generative model. Let the observations be $Y=\{(t_k,y_k)\}_{k=1}^{K}$, with $0=t_0<t_1<\cdots<t_K=T$. We introduce a continuous-time Markov latent state $Z_t\in\R^d$ to model these observations. LJMs therefore have three ingredients: a deterministic flow, a random event-jump mechanism, and an observation distribution.

Starting from $Z_0\sim p_0=\mathcal N(0,I)$, the state decays coordinatewise between events:
\begin{equation}
\dot Z_t=-\gamma\odot Z_t,\qquad
\Phi_\delta(z)=e^{-\gamma\delta}\odot z,\qquad \gamma\in(0,\infty)^d.
\label{eq:flow}
\end{equation}
Thus any event-free interval can be traversed in one closed-form operation. At time $t$, the pre-event state $z=Z_{t^-}$ determines the event intensity $\lambda_\theta(z)$ and the distribution of the mark $U\in\R^m$:
\begin{equation}
0<\lambda_{\min}\leq\lambda_\theta(z)\leq\bar\lambda_\theta<\infty,
\qquad p_\theta(u\mid z)=\mathcal N\!\left(u;\mu_\theta(z),\diag v_\theta(z)\right).
\label{eq:prior}
\end{equation}
The bounded rate gives finitely many events on every finite horizon. The state dependence makes the process responsive to its own history: an event can move it into a region with a different future rate or mark law. At each event time $\tau_n$, we sample $U_n\sim p_\theta(\cdot\mid Z_{\tau_n^-})$ and together call $(\tau_n, U_n)$ an \textit{event}. Each event updates state $Z_{\tau_n}=F_\theta(Z_{\tau_n^-},U_n)$ with the jump map $F_\theta$, where the pre-jump state $Z_{\tau_n^-}$ is obtained via the closed form deterministic flow $Z_{\tau_n^-} = \Phi_{\tau_n-\tau_{n-1}}(Z_{_{\tau_{n-1}}})$.

Although the generative model is valid for arbitrary jump maps, the state computation remains sequential when no structure is imposed on the network. In this work, we replace this sequential computation by parallel ``depth'' computation via our stacked block-triangular structure of the jump network $F_\theta$ which only reads the lower blocks of the sampled trajectory. This makes fast and parallel variational inference possible for LJMs with logarithmic temporal depth. We will discuss this in detail in Section \ref{sec:spvi}. 

Finally, the observation or emission model is $y_k\mid Z_{t_k}\sim p _\phi(y_k \mid Z_{t_k})$, evaluated on the observed coordinates. Our experiments assume a parametric form for the observation model, often a Gaussian with its mean parametrized by a neural network. Therefore, a predictive sample contains randomness from the inferred initial state, latent events, and observation model.

With measurable finite-valued updates and the bounded rate in Equation~\eqref{eq:prior}, these ingredients define a non-explosive, time-homogeneous Markov process. Appendix~\ref{app:theory} gives the construction of the generator and transition-semigroup argument.

\begin{figure}
    \centering
    \includegraphics[width=0.80\linewidth]{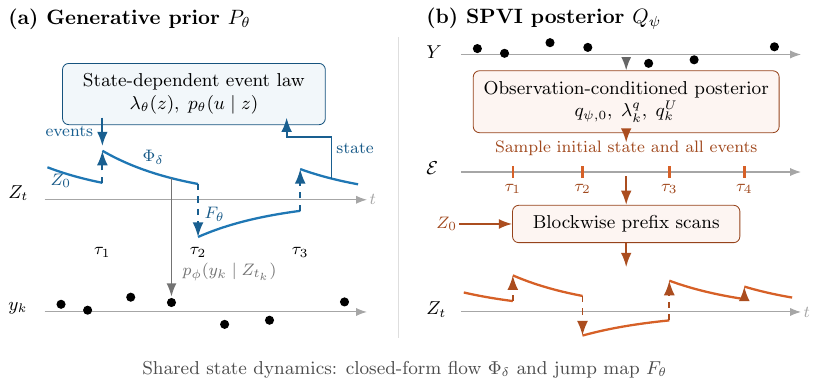}
    \caption{(a) \textbf{Generative prior}: the current latent state determines the intensity and mark distribution of subsequent events, requiring chronological simulation. (b) \textbf{SPVI posterior}: observations determine the initial-state distribution and event laws independently of the sampled latent state. The complete event configuration can therefore be sampled before constructing the trajectory through successive blockwise scans. Both processes share the closed-form flow \(\Phi_\delta\) and jump map \(F_\theta\).}
    \label{fig:prior-posterior}
    \vspace{-4mm}
\end{figure}
\label{sec:model}

\section{Scan-Parallel Variational Inference}
\label{sec:spvi}
In this section, we introduce \textit{Scan-Parallel Variational Inference} (SPVI), a framework that enables fast and scalable path-space training of event-driven state space models. Learning requires a distribution over the unobserved initial state and events. SPVI defines the variational posterior family as a \textit{controlled marked Poisson process} and proceeds in three steps: sample an observation-conditioned event configuration, construct its latent trajectory by scans, and compare that trajectory's event law with the generative prior via a tractable event-augmented path-space ELBO. We describe each of these ingredients one-by-one below.

\subsection{Variational family and event configuration generation}

In this section, we describe the variational family which makes the observation-conditioned posterior event sampling possible, enabling parallel state computation. It has three components: initial state head, event intensity head, and event mark head. All three heads are conditioned on a common recognition encoder, that uses causal attention over the observation sequence to produce observations embeddings \(h_{1:k} \in \mathbb R^{K \times d_h}\) in one single forward pass. Event parameters on \(I_k=(t_{k-1},t_k]\) are then computed from \(h_k\), which includes the right-endpoint observation \(y_k\), while the initial-state posterior uses the final context embedding \(h_K\). This defines an observation-conditioned variational posterior over the context trajectory, rather than a strictly online filtering procedure. At evaluation, the recognition network receives only the context observations; forecasting beyond the context uses the learned generative prior. 

For interval $I_k$ of length $\Delta_k$, recognition heads specify a rate $\lambda_{\psi,k}^q$ and Gaussian mark law $q_{\psi,k}^U=\mathcal N(\mu_{\psi, k}^q,\diag v_{\psi, k}^q)$. This yields $q_{\psi,0}(Z_0\mid Y)$ and the intervalwise event law
\begin{equation}
N_k\sim\Poi(\lambda_k^q\Delta_k),\qquad
\tau_{k,n}\mid N_k\sim\Unif(I_k),\qquad U_{k,n}\sim q_k^U.
\label{eq:posterior}
\end{equation}
Conditional on $Y$, the rate and mark parameters do not depend on the sampled latent state. All event times and marks can therefore be generated before state propagation. Denote their sorted collection by $\mathcal E=\{(\tau_n,U_n)\}_{n=1}^{N_T}$ and the law of $(Z_0,\mathcal E)$ by $Q_\psi$. This event driver therefore supplies parallelism without making the entire transition state independent. The shared jump map still reads the lower blocks of the sampled trajectory. We exploit that distinction next.

\subsection{Turn chronological updates into blockwise scans}

A fully general $F_\theta$ creates a sequential neural recurrence. We now introduce the stacked block-triangular structure of the jump network that retains nonlinear transitions while giving each block a simpler dependence on its own history. Split the state into $J$ ordered blocks, $z=(z^{(1)},\ldots,z^{(J)})$ and let $z^{(<j)}=(z^{(1)},\ldots,z^{(j-1)})$. For the $n^\text{th}$ event occurring at time $\tau_n$ with mark $U_{\tau_n}=u_n$, let the pre-jump state be $Z_{\tau_n^-}=\tilde z_n$ and the post-jump state be $Z_{\tau_n} = z_n=F_\theta(\tilde z_n, u_n)$. First, by the coordinate-wise decay dynamics between events, we have $$\tilde z_n^{(j)} = \Phi_{\delta_n}(z_{n-1}^{(j)}) = e^{-\gamma^{(j)}\delta_n}\odot z_{n-1}^{(j)}, \quad \text{where } \delta_n = \tau_n-\tau_{n-1}.$$ At jump time, we update blocks from bottom to top:
\begin{align}
(R_n^{(j)},B_n^{(j)})&=F_{\theta}^{(j)}(u_n,\tilde z_n^{(<j)},z_n^{(<j)}), \nonumber\\
z^{(j)}&=R_n^{(j)}\odot \tilde z^{(j)}+B_n^{(j)},\qquad
R_n^{(j)}\in(\rho_{\min},1)^{d_j}.
\nonumber
\end{align}
Here $\tilde z_n^{(<j)}$ and $z_n^{(<j)}$ are the lower blocks immediately before and after the event $(\tau_n, u_n)$. The retention $R^{(j)}$ determines how much memory survives: a retention near one preserves the previous state, while a small retention approximates a reset where the shift $B^{(j)}$ supplies new content. 

We emphasize that each $F^{(j)}_\theta$ can be any neural network, and the coefficients $R, B$ depend on lower blocks and the mark arbitrarily, but not on block $j$ itself. Therefore, the upper blocks respond to the realized lower-block trajectory giving a map that is nonlinear and expressive across blocks but affine in each block's own recurrent state. This structure enable parallel state computation while preserving expressivity. 

One can further combine the decay over $\delta_n=\tau_n-\tau_{n-1}$ with the event update and obtain the full recurrence as
\begin{equation}
z_n^{(j)}=A_n^{(j)}\odot z_{n-1}^{(j)}+B_n^{(j)},\qquad
A_n^{(j)}=R_n^{(j)}\odot e^{-\gamma^{(j)}\delta_n}.
\label{eq:recurrence}
\end{equation}
For the first block, every coefficient depends only on its mark and time gap. For the next block, every coefficient becomes available once the first block has been evaluated at all event times. The same argument applies successively to the remaining blocks. To evaluate one block, associate $(a,b)$ with the affine map $x\mapsto a\odot x+b$. Composing two such maps gives
\begin{equation}
(a_2,b_2)\star(a_1,b_1)
=(a_2\odot a_1,\ a_2\odot b_1+b_2).
\label{eq:scan}
\end{equation}
Composition is associative, so a prefix scan evaluates every state of the first block. Once that trajectory is available, the second block's coefficients can also be computed in one batch. Repeating this procedure evaluates all $J$ blocks.

\begin{proposition}
\label{prop:scan}
Conditional on the initial state and event stream, $J$ associative scans under Equation~\eqref{eq:scan} reproduce the chronological flow--jump trajectory. For fixed network architectures and block sizes, the scan stage has $O(JL)$ work and $O(J\log(1+L))$ temporal depth.
\end{proposition}
The proof uses affine composition within blocks and induction across blocks (Appendix~\ref{app:scanproof}). The construction uses only decays across nonnegative time gaps. Observation and quadrature states are then obtained from the preceding stored event state using Equation~\eqref{eq:flow}. Scans significantly  accelerate posterior state construction for identical models and inputs. We verify this empirically too: at batch-size $16$ and the training padding lengths, matched-model scans accelerate
posterior forward--backward computation by $125\times$ on cyclic jumps (Appendix~\ref{app:scan_benchmark}).

\subsection{A path-space ELBO via change of measure}
\label{sec:elbo}
The prior and posterior use the same flow and state transformation. Their difference lies in the initial distribution and in when and how events occur. Let $P_\theta$ denote the prior law of $(Z_0,\mathcal E)$. Under the domination and integrability conditions in Appendix~\ref{app:elbo}, the event-augmented ELBO is
\begin{equation}
\mathcal L(\theta,\phi,\psi)=
\E_{Q_\psi}\!\left[\sum_{k=1}^K\log p_\phi(y_k\mid Z_{t_k})\right]
-\KL(Q_\psi\Vert P_\theta)
\;\leq\;\log p_{\theta,\phi}(Y).
\label{eq:elbo}
\end{equation}
The marked-point-process likelihood ratio gives a local expression for the divergence:
\begin{align}
\KL(Q_\psi\Vert P_\theta)
&=\KL(q_{\psi,0}\Vert p_0)+\E_{Q_\psi}\int_0^T\kappa(t,Z_{t^-})\,dt,\label{eq:kl}\\
\kappa(t,z)
&=\lambda^q(t)\log\frac{\lambda^q(t)}{\lambda_\theta(z)}-\lambda^q(t)+\lambda_\theta(z)
+\lambda^q(t)\KL\!\left(q^U_t\Vert p_\theta(\cdot\mid z)\right).
\label{eq:kappa}
\end{align}
The first three terms compare event rates; the final term compares marks, weighted by their rate of occurrence. Both Gaussian KL terms are analytic. We estimate the remaining time integral using Monte Carlo random probes of the scanned trajectory. Thus the generative networks are evaluated at posterior states; full derivation is given in Theorem \ref{thm:elbo}.

\subsection{Training with hard events and forecasting with the prior}
To train posterior rates, we construct a homogeneous proposal process whose rate is determined by the maximum posterior intensity over that sequence's observation intervals. The envelope is recomputed from the current posterior rates at each training step. A candidate in interval \(k\) is accepted with probability \(\lambda_k^q/\bar\lambda_q\), where \(\bar\lambda_q=\sg\!\left(\max_k\lambda_k^q\right)+\rho,\) and $\rho$ is a small slack.
\begin{equation}
 g_i^{\rm h}=\ind\{\logit(p_k)+\zeta_i>0\},\quad
 g_i^{\rm s}=\sigmoid\!\left(\frac{\logit(p_k)+\zeta_i}{\tau_{\rm ST}}\right),\quad
 g_i=g_i^{\rm h}+g_i^{\rm s}-\sg(g_i^{\rm s}),
\label{eq:gate}
\end{equation}
where $\zeta_i$ is standard logistic noise. In the forward computation, an accepted proposal applies the original jump and a rejected proposal applies only the flow, using effective coefficients $1+ g_i(R_i-1)$ and $ g_iB_i$. Initial states and marks use Gaussian reparameterization. Algorithm~\ref{alg:train} in the appendix gives the implemented training procedure.

The experiments optimize a KL-weighted, event-regularized objective,
\begin{equation}
\mathcal J=
\E_{Q_\psi}\sum_k\log p_\phi(y_k\mid Z_{t_k})
-w_{\rm KL}\KL(Q_\psi\Vert P_\theta)
-\beta_{\rm sparse}\sum_k\lambda_k^q\Delta_k.
\label{eq:objective}
\end{equation}
The third term controls the activity of the learned latent representation. We evaluate the resulting tradeoff between predictive quality and the number of latent jump updates, without assuming a correspondence between learned events and physical events. $\mathcal J$ remains a lower bound on the ELBO in our setting since we use $w_{\rm KL}\geq1$ and $\beta_{\rm sparse}\geq0$. Training scans all candidates; posterior reconstruction at evaluation instead samples the accepted event process directly through Equation~\eqref{eq:posterior}.

Forecasting begins from a posterior sample of the last context state $Z_T$. Beyond $T$, the model uses the state-dependent prior. We generate proposals at rate $\bar\lambda_\theta$, propagate by the closed-form flow, and accept each proposal with probability $\lambda_\theta(Z_{t^-})/\bar\lambda_\theta$. Accepted marks are drawn from $p_\theta(\cdot\mid Z_{t^-})$. Repeating this procedure and sampling the emission model produces predictive paths, from which we calculate means and quantile intervals. Appendix~\ref{app:algorithms} gives the reconstruction and forecasting algorithms.
\section{Experiments}
\label{sec:experiments}

In this section, we show that LJM achieves better or competitive predictive quality on four synthetic systems and two real datasets while taking fewer dynamics updates than the numerical solver baselines. Latent ODE, Latent SDE, and NCDSSM require sequential solvers during training and inference both; SDE Matching matches marginals during training however inference requires sequential solvers; Neural Stochastic Flow (NSF) is  the solver-free baseline, but requires one update per observation interval. Appendix~\ref{app:metrics} gives definitions of the evaluated metrics and detailed training protocols.

\subsection{Synthetic Systems: smooth rotations, scheduled jumps and PDMPs}
\label{sec:synthetic}
\begin{table}[t]
\centering
\caption{Synthetic forecasting benchmarks. The reported numbers are mean and range over three independent runs of the whole study. NFE is dynamics evaluations per context trajectory; in the columns marked $\times$, a value $K \times$ means LJM is $K$ times \textit{cheaper} than the baseline.}
\label{tab:synthetic_budget}
\small
\setlength{\tabcolsep}{4pt}
\begin{tabular}{l ccc l ccc c}
\toprule
& \multicolumn{3}{c}{LJM (ours)} & \multicolumn{5}{c}{Best baseline at a similar budget} \\
\cmidrule(lr){2-4}\cmidrule(lr){5-9}
System & CRPS $\downarrow$ & RMSE $\downarrow$ & NFE $\downarrow$ & Method & CRPS $\downarrow$ & RMSE $\downarrow$ & NFE $\times$ & Train time $\times$ \\
\midrule
Cyclic jumps & $\underset{\text{\tiny$(0.0032)$}}{\mathbf{0.0679}}$ & $\underset{\text{\tiny$(0.0036)$}}{\mathbf{0.1261}}$ & 89 & Latent ODE & $\underset{\text{\tiny$(0.0010)$}}{0.0729}$ & $\underset{\text{\tiny$(0.0015)$}}{0.1276}$ & 2.0$\times$ & 7$\times$ \\
Irregular PDMP & $\underset{\text{\tiny$(0.0012)$}}{\mathbf{0.2244}}$ & $\underset{\text{\tiny$(0.0017)$}}{\mathbf{0.5210}}$ & 19 & Latent SDE & $\underset{\text{\tiny$(0.0036)$}}{0.2410}$ & $\underset{\text{\tiny$(0.0118)$}}{0.5467}$ & 2.6$\times$ & 10$\times$ \\
Regular PDMP & $\underset{\text{\tiny$(0.0009)$}}{\mathbf{0.2509}}$ & $\underset{\text{\tiny$(0.0022)$}}{0.5995}$ & 19 & SDE Matching & $\underset{\text{\tiny$(0.0009)$}}{0.2567}$ & $\underset{\text{\tiny$(0.0041)$}}{\mathbf{0.5945}}$ & 2.3$\times$ & 1$\times$ \\
Damped rotation & $\underset{\text{\tiny$(0.0056)$}}{0.1012}$ & $\underset{\text{\tiny$(0.0225)$}}{0.1693}$ & 205 & NCDSSM-NL & $\underset{\text{\tiny$(0.0017)$}}{\mathbf{0.0743}}$ & $\underset{\text{\tiny$(0.0027)$}}{\mathbf{0.1274}}$ & 1.7$\times$ & 86$\times$ \\
\bottomrule
\end{tabular}
\end{table}

The synthetic suite (see Figure \ref{fig:illustration}) spans damped rotations, cyclic schecduled jumps, and irregularly and regularly observed PDMPs. Each evaluation uses 1,000 trajectories, 44 context observations, eight forecast observations, and 128 predictive paths. Table~\ref{tab:synthetic_budget} reports LJM at the lowest-CRPS event penalty $\beta_\text{sparse}$ (selected on an independent validation set: $0.3$ for cyclic jumps, $1$ for both PDMPs, $0.1$ for damped rotation) against the strongest baseline at the closest budget it can reach. Budget match is approximate because solvers cannot take fewer than one step per observation interval but the baseline's budget is always chosen to be more than LJM's.

Across the synthetic benchmarks, LJM performs the best accross all event driven benchmarks while remaining competitive on rotation with $86\times$ faster training and $1.7\times$ faster inference.
Future-event suppression at inference leaves the sparsest cyclic fits nearly unchanged but degrades the PDMP and rotation fits (Appendix~\ref{app:future_events}).

\subsection{REFIT electrical load measurement data}
\label{sec:refit}
We evaluate long-context forecasting with $512-2,048$ observations for whole-house electrical demand for the 18 households of the REFIT Electrical Load Measurements release \citep{murray2017electrical}. The raw distribution of per-appliance readings is aggregated into
one-minute arithmetic means over $[t-60\,\mathrm{s}, t)$, and timestamped at the closing edge. The model receives only past aggregate power and elapsed timestamps. Power is scaled per household as $x = (P + 1\,\mathrm{W})/s_h$ with
$s_h$ the $0.99$ quantile of $P + 1\,\mathrm{W}$ over that household's unique
training minutes. Given a context of $512$, $1024$ or $2048$ minutes the task is to
forecast power at the every minute for the next $30$ minutes. We focus on scalable training baselines (SDE Matching and Latent NSF) for REFIT since sequential solver-based methods  are computationally prohibitive to train at long contexts under our experimental budget.

As shown in Table \ref{tab:refit_contexts}, LJM performs best on all three metrics at every context length showing that the advantage is not an artefact of one scoring rule: it holds for a
per-minute proper score, for a plain squared error, and for the aggregated
energy over the horizon, which is the quantity a demand-response application
actually consumes. Against
the stronger baseline, SDE Matching, it reduces normalised CRPS by 11--12\% and
horizon-total energy CRPS by 12\%; against the Neural Stochastic Flow the
margins are 13--17\% and 14--20\%. The model remains well-calibrated with sharp prediction intervals for all context lengths. The reason is that aggregate household load is
dominated by appliances switching on and off, which is a piecewise-continuous
signal with abrupt transitions. Therefore, a model whose latent path can only diffuse must
absorb those transitions into its diffusion term, widening its predictive
distribution everywhere to cover jumps that occur in a few places.
\begin{table}[t]
\centering\small
\setlength{\tabcolsep}{4pt}
\caption{REFIT aggregate load forecasting using 256 predictive trajectories per test window. Context is the number of observed points. CRPS$/s_h$ normalises by the household scale. Energy CRPS scores total consumption over the 30-minute forecast horizon.}
\label{tab:refit_contexts}
\begin{tabular}{@{}llrrrrr@{}}
\toprule
Context & Method & CRPS$/s_h\downarrow$ & RMSE (W)$\downarrow$ & $\text{E-CRPS}_\text{ \tiny(kWh)}$$\downarrow$ & 90\% cov. & Width (W)\\
\midrule
\multirow{3}{*}{512}
 & LJM (ours) & \textbf{0.0495}$_{\pm0.0004}$ & \textbf{551.7}$_{\pm2.9}$ & \textbf{0.0630}$_{\pm0.0005}$ & 0.925$_{\pm0.005}$ & 858$_{\pm33}$\\
 & SDE Matching & 0.0544$_{\pm0.0011}$ & 566.9$_{\pm1.2}$ & 0.0705$_{\pm0.0007}$ & 0.876$_{\pm0.017}$ & 804$_{\pm120}$\\
 & Latent NSF$^{\dagger}$ & 0.0588$_{\pm0.0018}$ & 601.5$_{\pm21.3}$ & 0.0779$_{\pm0.0031}$ & 0.877$_{\pm0.012}$ & 918$_{\pm322}$\\
\midrule
\multirow{3}{*}{1024}
 & LJM (ours) & \textbf{0.0492}$_{\pm0.0004}$ & \textbf{551.4}$_{\pm0.2}$ & \textbf{0.0628}$_{\pm0.0005}$ & 0.926$_{\pm0.008}$ & 926$_{\pm40}$\\
 & SDE Matching & 0.0568$_{\pm0.0006}$ & 573.7$_{\pm2.0}$ & 0.0730$_{\pm0.0010}$ & 0.883$_{\pm0.003}$ & 937$_{\pm161}$\\
 & Latent NSF$^{\dagger}$ & 0.0599$_{\pm0.0015}$ & 589.3$_{\pm9.2}$ & 0.0794$_{\pm0.0030}$ & 0.872$_{\pm0.013}$ & 740$_{\pm135}$\\
\midrule
\multirow{3}{*}{2048}
 & LJM (ours) & \textbf{0.0495}$_{\pm0.0005}$ & \textbf{551.5}$_{\pm1.3}$ & \textbf{0.0632}$_{\pm0.0008}$ & 0.925$_{\pm0.011}$ & 892$_{\pm29}$\\
 & SDE Matching & 0.0555$_{\pm0.0008}$ & 567.0$_{\pm1.5}$ & 0.0709$_{\pm0.0006}$ & 0.918$_{\pm0.008}$ & 1014$_{\pm97}$\\
 & Latent NSF$^{\dagger}$ & 0.0590$_{\pm0.0014}$ & 577.0$_{\pm2.6}$ & 0.0765$_{\pm0.0021}$ & 0.888$_{\pm0.015}$ & 882$_{\pm100}$\\
\bottomrule
\end{tabular}
\end{table}
. 

Table \ref{tab:refit_beta} further varies the sparsity penalty and shows that LJM performs better than both baselines at much fewer dynamics evaluations. Its coverage stays
between $0.909$ and $0.929$ throughout and is closest to nominal at
$\beta_{\mathrm{sparse}}=2$, where the interval width is also narrowest at
$735$~W, so the sparser models are not achieving their coverage by widening.

\subsection{USHCN climate data}
\label{sec:climate}
\begin{table}[t]
\centering
\small
\setlength{\tabcolsep}{3pt}
\caption{USHCN test forecasts: The numbers show mean $\pm$ sd over five station folds with 128 predictive paths.}
\label{tab:climate}
\begin{tabular}{@{}lcccrc@{}}
\toprule
Model & Forecast MSE $\downarrow$ & CRPS $\downarrow$ & 90\% coverage & NFE $\downarrow$ & Train time/fold $\downarrow$ \\
\midrule
LJM (ours) & $\underline{0.360}_{\pm0.069}$ & $\mathbf{0.258}_{\pm0.024}$ & $0.837_{\pm0.030}$ & 461 & \textbf{7.5 min} \\
NCDSSM-NL & $\mathbf{0.352}_{\pm0.055}$ & $0.305_{\pm0.013}$ & $0.968_{\pm0.006}$ & 6740 & 3.6 h \\
Latent SDE & $0.414_{\pm0.078}$ & $0.306_{\pm0.040}$ & $0.592_{\pm0.127}$ & 1685 & 21 h \\
Latent ODE & $0.451_{\pm0.107}$ & $0.371_{\pm0.042}$ & $0.284_{\pm0.033}$ & 6740 & 11.5 h \\
SDE Matching & $0.433_{\pm0.106}$ & $\underline{0.259}_{\pm0.030}$ & $0.958_{\pm0.018}$ & 1685 & 33 min \\
Latent NSF & $0.829_{\pm0.162}$ & $0.477_{\pm0.058}$ & $0.605_{\pm0.049}$ & $\mathbf{422}$ & \underline{19 min} \\
\bottomrule
\end{tabular}
\end{table}

We evaluate five station folds, each with 110--112 test stations. Context extends through $t=150$, and the next three observations are predicted; only recorded coordinates are scored. LJM and all baselines use 450 epochs except NCDSSM-NL that uses 200 optimizer steps at batch size 100.

LJM's mean MSE is $0.360$, compared with $0.352$ for NCDSSM-NL, with training times of $7.5$ minutes and $3.6$ hours per fold: a roughly $29\times$ difference; see Table \ref{tab:climate} and Appendix~\ref{app:real} for more details.


\section{Limitations and Conclusion}
\label{sec:limitations}
\textbf{Limitations.} Although latent events provide a probabilistic representation of the observed dynamics, but their times and marks need not be uniquely determined by sparse, noisy measurements. Several smaller jumps may explain the observations as well as a few larger jumps, while a flexible decoder introduces additional ambiguity in the latent representation. The expected-count penalty favors parsimonious event configurations but does not establish their identifiability. Our experiments evaluate predictive performance and event activity, rather than recovery of ground-truth events. Second, the variational posterior samples all events in parallel before state computation, hence posterior event times cannot react to past sampled states. Lastly, state-dependent forecasting remains chronological in our framework. Extending our framework for parallelized forecasting remains an open direction for future work.



\textbf{Conclusion.} We introduced latent jump models, a finite-activity latent PDMP with state-dependent event laws, learned memory retention and exact between-event propagation, together with an event-augmented path-space ELBO derived by a change of measure. SPVI replaces the length-$L$ chain of event updates with $J$ successive scans of $O(J \log(1+L))$ depth while keeping nonlinear interactions across blocks. On six synthetic systems and two real datasets, LJMs matched the forecast error of the strongest evaluated baselines on event-driven tasks with fewer latent jump updates than the baselines requiring numerical solvers, and trained one to two orders of magnitude faster than the sequential ones.

\label{mainend}
\clearpage
\subsection*{AI use statement}

In this work, we used generative AI tools for implementation and editing the manuscript. LLM-generated code was verified and tested for correctness. We take responsibility for the final content of this work, including text, claims or artifacts produced with the aid of generative AI.
\bibliographystyle{abbrvnat}
\bibliography{references_scan_parallel,iclr2027_conference}

\clearpage
\appendix

\section{Mathematical details}
\label{app:theory}
\subsection{Finite activity, exact event simulation, and the Markov property}
The generative state has no explicit dependence on absolute time. Its flow is $\Phi_\delta(z)=e^{-\gamma\delta}\odot z$, its event intensity is $\lambda_\theta(z)$, and its mark kernel is $p_\theta(du\mid z)$. We make the following assumptions explicit.

\begin{assumption}[Generative regularity]
\label{ass:generative}
The initial law is a probability measure on $\R^d$; the rate is Borel measurable with $0<\lambda_{\min}\leq\lambda_\theta(z)\leq\bar\lambda_\theta<\infty$; the mark distribution is a Borel probability kernel; and the jump map is measurable and finite-valued at every finite state and mark.
\end{assumption}

\begin{proposition}[Finite-activity generative process]
Under Assumption~\ref{ass:generative}, the flow--jump construction defines a non-explosive time-homogeneous Markov process. For a horizon $T$, its accepted event count satisfies
\begin{equation}
N_T\preceq_{\rm st}\Poi(\bar\lambda_\theta T).
\end{equation}
Its transition operators $P_t f(z)=\E[f(Z_t)\mid Z_0=z]$ satisfy $P_{s+t}=P_sP_t$.
\end{proposition}
\begin{proof}
Generate a homogeneous Poisson proposal process of rate $\bar\lambda_\theta$. At each proposal time, propagate the state with $\Phi$, accept with probability $\lambda_\theta(z)/\bar\lambda_\theta$, and, upon acceptance, draw a mark from $p_\theta(\cdot\mid z)$ and apply $F_\theta$. The proposal count is finite almost surely on a finite horizon, and every update produces a finite state. Accepted events are a subset of proposals, giving the count bound. The accepted conditional intensity and mark law equal those specified by the model. Future evolution depends only on the current state and fresh proposal, acceptance, and mark randomness. Since none of these transition rules depends on absolute time, the process is time homogeneous and Markov. Applying conditional expectation at time $s$ gives $P_{s+t}f=P_s(P_tf)$.
\end{proof}
For smooth test functions with locally integrable jump terms, the compensated jump formula gives the extended generator
\begin{equation}
\mathcal A_\theta f(z)=\langle-\gamma\odot z,\nabla f(z)\rangle
+\lambda_\theta(z)\int\!\left[f(F_\theta(z,u))-f(z)\right]p_\theta(du\mid z).
\label{eq:generator}
\end{equation}
The first term describes event-free motion; the second describes random state transformations. Event-free intervals use the exact flow, and thinning samples the specified event law without a time grid. This is exact simulation of the parameterized process, with ordinary floating-point arithmetic used to evaluate its maps.

\subsection{Why the triangular recurrence admits a scan}
\label{app:scanproof}
For three affine maps, both parenthesizations of Equation~\eqref{eq:scan} give
\begin{equation}
\left(a_3\odot a_2\odot a_1,\;
 a_3\odot a_2\odot b_1+a_3\odot b_2+b_3\right).
\end{equation}
Thus $\star$ is associative and has identity $(\mathbf 1,\mathbf 0)$. If $(\alpha_i,\beta_i)$ is the prefix composition through event $i$, then $z_i=\alpha_i\odot z_0+\beta_i$. An implementation that assumes zero initial state can instead replace $b_1$ with $b_1+A_1\odot z_0$.

\begin{proof}[Proof of Proposition~\ref{prop:scan}]
In block 1, the marks, gates, and inter-event gaps determine every coefficient before the recurrence is evaluated. Prefix composition therefore gives exactly the chronological block-1 states. Suppose blocks $1,\ldots,j-1$ have been recovered. Their pre-event states are obtained by applying the next decay to the preceding post-event states; their post-event states are already available. Together with the marks, these are precisely the inputs to $F_\theta^{(j)}$. All block-$j$ coefficients can therefore be evaluated in a batch. A second affine prefix calculation gives exactly its chronological states. Induction proves the result for all blocks. Each work-efficient scan has linear work and logarithmic span; there are $J$ sequential block scans.
\end{proof}
This argument also covers rejected candidates: when $g_i=0$, the event map is the identity and the coefficient contains only the decay. Identity padding uses $(A_i,b_i)=(1,0)$ and does not alter any prefix. A trajectory with no events is simply $Z_t=\Phi_t(Z_0)$.

To answer a query at time $t$, find $i=\max\{n:r_n\leq t\}$, with $r_0=0$ and stored $Z_0$, and return $\Phi_{t-r_i}(Z_{r_i})$. The implementation uses post-event states for observation queries and pre-event states for KL probes on exact ties. Under continuous event times, a tie with a fixed observation time occurs with probability zero.

\subsection{The event-augmented likelihood ratio and ELBO}
\label{app:elbo}
Write $\lambda^q(t)=\lambda_k^q$ and $q_t^U=q_k^U$ on $I_k$. Conditional on the observations, the variational event process is a marked inhomogeneous Poisson process. Its initial state is sampled independently of its event configuration under the specified variational family.

\begin{assumption}[Domination and finite entropy]
\label{ass:entropy}
On $[0,T]$, $q_{\psi,0}\ll p_0$; the prior and posterior event processes are non-explosive with positive finite rates; and $q_t^U\ll p_\theta(\cdot\mid Z_{t^-})$ for almost every time and $Q_\psi$-almost every reachable state. The expected absolute initial, event-rate, and mark log-ratios, both compensator integrals, and the absolute observation log likelihood are finite.
\end{assumption}
The positive Gaussian densities and bounded rates used in the implementation provide the domination conditions. The integrability requirement concerns their values along the sampled state process.

\begin{theorem}[Event-augmented change of measure] \label{thm:elbo}
Under Assumptions~\ref{ass:generative} and~\ref{ass:entropy}, $Q_\psi\ll P_\theta$ on the space of $(Z_0,\mathcal E)$, and
\begin{align}
\log\frac{dQ_\psi}{dP_\theta}
={}&\log\frac{q_{\psi,0}(Z_0\mid Y)}{p_0(Z_0)}
+\sum_{\tau_n\leq T}\left[
\log\frac{\lambda^q(\tau_n)}{\lambda_\theta(Z_{\tau_n^-})}
+\log\frac{q^U_{\tau_n}(U_n)}{p_\theta(U_n\mid Z_{\tau_n^-})}\right]\notag\\
&-\int_0^T\left[\lambda^q(t)-\lambda_\theta(Z_{t^-})\right]dt.
\label{eq:rn}
\end{align}
Taking expectations yields Equations~\eqref{eq:kl}--\eqref{eq:kappa} and the ELBO in Equation~\eqref{eq:elbo}.
\end{theorem}
\begin{proof}
Fix the event count and use ordered event times and marks as coordinates. The generative density is the initial density times
\begin{equation}
\exp\!\left\{-\int_0^T\lambda_\theta(Z_{t^-})dt\right\}
\prod_{n=1}^{N_T}\lambda_\theta(Z_{\tau_n^-})p_\theta(U_n\mid Z_{\tau_n^-}).
\end{equation}
The posterior density has the same form with $q_{\psi,0}$, $\lambda^q$, and $q_t^U$. The shared deterministic maps recover the state from these coordinates. Domination allows division of the two densities; taking logarithms gives Equation~\eqref{eq:rn}. Under $Q_\psi$, the compensation identity gives
\begin{equation}
\E_{Q_\psi}\sum_{\tau_n\leq T}a(\tau_n,Z_{\tau_n^-},U_n)
=\E_{Q_\psi}\int_0^T\lambda^q(t)\int a(t,Z_{t^-},u)q_t^U(du)dt.
\end{equation}
Applying it to the event log-ratios yields the rate and mark terms in Equation~\eqref{eq:kappa}. Finally, applying Jensen's inequality to the event-augmented marginal likelihood gives Equation~\eqref{eq:elbo}.
\end{proof}
The comparison is performed before mapping events into state paths, so no jump-map Jacobian enters the objective. If several event configurations produce the same state path, the augmented KL can exceed the KL between the induced state-path laws; the augmented objective is still a valid lower bound.

For diagonal Gaussians, the mark KL at a state $z$ is
\begin{equation}
\KL(q_k^U\Vert p_\theta(\cdot\mid z))
=\frac12\sum_{a=1}^m\left[
\log\frac{v_{\theta,a}(z)}{v^q_{k,a}}
+\frac{v^q_{k,a}+(\mu^q_{k,a}-\mu_{\theta,a}(z))^2}{v_{\theta,a}(z)}-1\right].
\end{equation}
For $M$ quadrature probes per interval, draw independent $V_{k,s}\sim\Unif(0,1)$ and use
\begin{equation}
\xi_{k,s}=t_{k-1}+\frac{\Delta_k}{M}(s-1+V_{k,s}),\qquad
\widehat K_{\rm dyn}=\sum_{k=1}^{K}\frac{\Delta_k}{M}\sum_{s=1}^{M}\kappa(\xi_{k,s},Z_{\xi_{k,s}^-}).
\label{eq:quadrature}
\end{equation}
Conditional on the sampled path, this stratified estimate is unbiased for the time integral. Increasing $M$ changes the objective estimator, not the sampled path.
\section{Algorithms and implementation}
\label{app:algorithms}
The algorithms below use $F_\theta$ for the entire post-event state transformation. An equivalent residual implementation writes $G_\theta(z,u)=F_\theta(z,u)-z$ and updates $z\leftarrow z+G_\theta(z,u)$. The pseudocode describes one trajectory; the implementation vectorizes over trajectories and predictive samples.

\subsection{Training}
Training uses a posterior proposal envelope, hard-forward events, and a straight-through derivative. For a proposal with acceptance probability $p$, the forward value of Equation~\eqref{eq:gate} is Bernoulli$(p)$. The derivative with respect to its logit is
\begin{equation}
\frac{\partial g_i}{\partial\logit(p_i)}
=\frac{g_i^{\rm s}(1-g_i^{\rm s})}{\tau_{\rm ST}}.
\end{equation}
This is the surrogate used for event occurrence. Explicit posterior-rate terms in the KL and sparsity penalty are differentiated directly. Reparameterized initial states and Gaussian marks carry ordinary derivatives. Proposed candidates and their times are treated as constant in the backward pass.

\begin{algorithm}[H]
\caption{SPVI training step with hard-forward thinning}
\label{alg:train}
\begin{algorithmic}[1]
\Require observations $Y$ on $[0,T]$, quadrature count $M$,
         small slack $\rho>0$
\State Encode $Y$ and compute $q_{\psi,0}$ and
       $(\lambda_k^q,\mu_k^q,v_k^q)_{k=1}^K$
\State Set the sequence-specific proposal rate
       $\bar\lambda_q \gets
       \sg\!\left(\max_{1\leq k\leq K}\lambda_k^q\right)+\rho$
\State Sample $Z_0\sim q_{\psi,0}$ by Gaussian reparameterization
\State Draw $L\sim\Poi(\bar\lambda_q T)$
\State Draw $r_{1:L}\overset{\rm iid}{\sim}\Unif(0,T)$
       and sort in increasing order
\ForAll{candidates $i$ in parallel}
  \State Find its interval $k$ and set
         $p_i=\lambda_k^q/\bar\lambda_q$
  \State Draw a mark
         $U_i=\mu_k^q+\sqrt{v_k^q}\odot\varepsilon_i$,
         $\varepsilon_i\sim\mathcal N(0,I)$
  \State Draw logistic noise and construct $g_i$
         using Equation~\eqref{eq:gate}
\EndFor
\State Compute the candidate states with
       \Call{BlockScan}{$Z_0,r,U,g$}
\State Query states at observation times and
       stratified quadrature times
\State Compute the observation log likelihood and
       the initial-state Gaussian KL
\State Estimate the dynamic KL with
       Equation~\eqref{eq:quadrature}
\State Form a sample of $\mathcal J$
       in Equation~\eqref{eq:objective}
\State Backpropagate its negative, clip the gradient,
       and update with AdamW
\end{algorithmic}
\end{algorithm}

\begin{algorithm}[H]
\caption{Layerwise state construction shared by training and reconstruction}
\label{alg:scan}
\begin{algorithmic}[1]
\Function{BlockScan}{$Z_0,r_{1:L},U_{1:L},g_{1:L}$}
\If{$L=0$}\State \Return the initial state and the event-free flow\EndIf
\State Set $r_0=0$, $H_0=Z_0$, and compute $D_i=e^{-\gamma(r_i-r_{i-1})}$
\For{$j=1,\ldots,J$}
  \State Read known lower pre-event states $D_i^{(<j)}\odot H_{i-1}^{(<j)}$
  \State Read known lower post-event states $H_i^{(<j)}$
  \State Compute every $(R_i^{(j)},B_i^{(j)})$ with $F_\theta^{(j)}$ in parallel
  \State $A_i^{(j)}\gets[\mathbf 1+g_i(R_i^{(j)}-\mathbf 1)]\odot D_i^{(j)}$, $b_i^{(j)}\gets g_i B_i^{(j)}$
  \State Set $(A_i^{(j)},b_i^{(j)})=(1,0)$ on padding positions
  \State Compute every affine prefix under $\star$ and apply it to $Z_0^{(j)}$
\EndFor
\State \Return all post-event states, with $Z_0$ prepended
\EndFunction
\end{algorithmic}
\end{algorithm}

\subsection{Posterior reconstruction and generative forecasting}
At evaluation, the recognition network is used only on the context. Each predictive path samples its own initial state and context event configuration. The implementation repeats the context encoding across the packed predictive samples. Equation~\eqref{eq:posterior} directly generates posterior events without an envelope; all of these events are accepted. Future events use the generative rate and mark law.

\begin{algorithm}[H]
\caption{Context reconstruction followed by state-dependent prior forecasting}
\label{alg:forecast}
\begin{algorithmic}[1]
\Require context $Y$ through $T$; future query times $t^f_{1:F}$; prior cap $\bar\lambda_\theta$
\State Encode the context and sample $Z_0\sim q_{\psi,0}$
\ForAll{context intervals $k$ in parallel}
  \State Sample $N_k\sim\Poi(\lambda_k^q\Delta_k)$, uniform event times, and Gaussian marks
\EndFor
\State Sort and pack the events; apply \Call{BlockScan}{$Z_0,r,U,\mathbf 1$}
\State Query $Z_T$ and initialize $z\gets Z_T$, $s\gets0$, $H\gets t_F^f-T$; store $(0,Z_T)$
\State Draw $M\sim\Poi(\bar\lambda_\theta H)$ and sort $c_{1:M}\overset{\rm iid}{\sim}\Unif(0,H)$
\For{$i=1,\ldots,M$ in chronological order}
  \State $z\gets\Phi_{c_i-s}(z)$; $s\gets c_i$
  \State Draw $A_i\sim\operatorname{Bernoulli}(\lambda_\theta(z)/\bar\lambda_\theta)$
  \If{$A_i=1$}
    \State Draw $U_i\sim p_\theta(\cdot\mid z)$ and set $z\gets F_\theta(z,U_i)$
  \EndIf
  \State Store the post-candidate state and time
\EndFor
\State Query future states from the preceding stored states using the closed-form flow
\State Decode means and draw observation samples; repeat to form a predictive ensemble
\end{algorithmic}
\end{algorithm}

\subsection{Recorded synthetic configuration}
\label{app:implementation}
Table~\ref{tab:hyper} summarizes the synthetic implementation settings. The recognition network is a causal pre-normalized transformer with zero dropout. Posterior event rates use a bounded log-rate link,
\begin{equation}
\lambda_k^q=\exp\!\left(a+(b-a)\sigmoid(\ell_\psi(h_k))\right),\qquad
 a=\log10^{-3},\quad b=\log10.
\end{equation}
Posterior mark and initial-state variances use an exponentiated log-variance clipped to $[-7,3]$. The initial-state head reads the final context embedding. Retention gates lie between $10^{-2}$ and one. The reported total trainable parameter counts are 1,152,436 for LJM, 854,195 for Latent ODE, 1,105,811 for Latent SDE, and 288,659 for the synthetic NCDSSM model.

\begin{table}[t]
\centering\small
\caption{Synthetic implementation settings. The KL weight takes values in the listed set; the event-regularization weight is swept separately.}
\label{tab:hyper}
\begin{tabular}{@{}lr@{}}
\toprule Setting & Value\\\midrule
Latent dimension / ordered block sizes & $8$ / $(4,4)$\\
Mark dimension / observation dimension & $2$ / $2$\\
Training batch size / observations per window & $512$ / $52$\\
Context / forecast observations & $44$ / $8$\\
Predictive samples / evaluation batch size & $128$ / $16$\\
Posterior rate range & $(10^{-3},10)$\\
Prior rate cap & $10$ / $15$\\
Straight-through temperature & $1$\\
Quadrature probes per interval & $4$\\
KL weight & $\{2,5\}$\\
Posterior sparsity weights & $\{0,0.03,0.1,0.3,1\}$\\
Prior-rate penalty & $0$\\
AdamW momentum coefficients & $(0.9,0.99)$\\
Weight decay / gradient-norm cap & $10^{-5}$ / $1$\\
Whole-window candidate allocation guard & $4096$\\
\bottomrule
\end{tabular}
\end{table}

The training likelihood is a per-trajectory sum, averaged over the batch. There is one Monte Carlo trajectory per training example. The learning-rate implementation combines linear warmup from $0.2$ of the peak rate with cosine decay. Candidate streams are padded to a power of two for the scan implementation.

\section{Synthetic data-generating processes}
\label{sec:datasets}

All six synthetic datasets share a two-dimensional latent state
$z\in\mathbb R^{2}$ and a two-dimensional observation $y\in\mathbb R^{2}$.
Each trajectory is observed at $52$ times, of which the first $44$ form the
context window and the last $8$ the forecast horizon. The observation model is
common to all six: writing $z(t_k)$ for the latent state at the $k$-th
observation time,
\begin{equation}
y_k \;=\; z(t_k) \;+\; b \;+\; \sigma\,\varepsilon_k,
\qquad \varepsilon_k\stackrel{\text{iid}}{\sim}\mathcal N(0,I_2),
\label{eq:obs}
\end{equation}
with an identity emission map, an additive offset $b\in\mathbb R^{2}$ that is
zero for every dataset except \textsc{Irregular PDMP}, and isotropic Gaussian
noise of scale $\sigma$. The models are trained with their decoder scale pinned
to this $\sigma$, so \eqref{eq:obs} is known exactly at training time. Every
per-trajectory quantity below is drawn independently for each trajectory from a
single seeded stream.

Two of the datasets contain no jumps and serve as negative controls; two contain
jumps at deterministic, periodic times; and two are genuine piecewise-deterministic
Markov processes with stochastic event times.

\subsection{No-jump controls}

Both controls evolve by a linear map applied once per unit time step
$\delta=1$, on the regular grid $t_k=k\delta$, $k=1,\dots,52$, and both start
from the same ring distribution
\begin{equation}
z_0 \;=\; \rho\,(\cos\varphi,\ \sin\varphi)^{\!\top},
\qquad \rho\sim\mathcal U(0.5,\,2.0),
\qquad \varphi\sim\mathcal U(0,\,2\pi).
\label{eq:ring}
\end{equation}
Write $R(\vartheta)=\left(\begin{smallmatrix}\cos\vartheta&-\sin\vartheta\\
\sin\vartheta&\cos\vartheta\end{smallmatrix}\right)$ for the planar rotation.

\paragraph{Deterministic rotation.}
A noiseless rigid rotation: the latent path is fully determined by
$(\theta,z_0)$ and all randomness in $y$ comes from \eqref{eq:obs}.
\begin{equation}
z_k \;=\; R(\theta)\,z_{k-1},
\qquad \theta\sim\mathcal U(0.25,\,0.55),
\qquad \sigma=0.10,\quad b=0.
\end{equation}

\paragraph{Damped rotation.}
A contracting rotation \emph{with process noise}, so unlike the deterministic
control its latent path is itself stochastic:
\begin{equation}
z_k \;=\; e^{\alpha\delta}\,R(\omega\delta)\,z_{k-1} \;+\; \eta_k,
\qquad \eta_k\stackrel{\text{iid}}{\sim}\mathcal N\!\big(0,\ \varsigma^{2}I_2\big),
\end{equation}
with $\alpha\sim\mathcal U(-0.004,-0.002)$, $\omega\sim\mathcal U(0.4,\,0.8)$,
process-noise scale $\varsigma=0.02$, and $\sigma=0.10$, $b=0$.

\subsection{Deterministic periodic jumps: the two cyclic datasets}

Both cyclic variants use the same class and differ only in whether the decay
rates are resampled per trajectory. Between observation times the latent decays
coordinatewise, and every $m=5$ steps a mark is added, cycling through four
directions with period four. On the grid $t_k=k\delta$, $\delta=1$, initialised
from the ring law \eqref{eq:ring}, the recursion is
\begin{equation}
z_k \;=\; F\,z_{k-1} \;+\; \mathbf 1\!\left[k \bmod m = 0\right]\,
u_{\,n(k)\bmod 4},
\qquad
F=\operatorname{diag}\!\left(e^{a_1\delta},\ e^{a_2\delta}\right),
\end{equation}
where $n(k)=\lfloor k/m\rfloor$ counts the events so far, the decay acts
\emph{before} the jump, and the four marks are
\begin{equation}
u_1=\begin{pmatrix}c_1\\0\end{pmatrix},\quad
u_2=\begin{pmatrix}0\\c_2\end{pmatrix},\quad
u_3=\begin{pmatrix}-c_1\\0\end{pmatrix},\quad
u_0=\begin{pmatrix}0\\-c_2\end{pmatrix},
\end{equation}
with amplitudes $c_1\sim\mathcal U(0.9,\,1.5)$ and $c_2\sim\mathcal U(0.6,\,1.0)$
drawn once per trajectory. With $52$ steps and $m=5$ this gives exactly $10$
events per trajectory, at $k\in\{5,10,\dots,50\}$, at times known in advance.
Two consequences are worth stating. The amplitudes $c_1,c_2$ are drawn once per
trajectory rather than once per event, so conditional on a trajectory the whole
mark sequence is deterministic and the jump channel carries no aleatoric noise.
And because events fall exactly on observation times with the decay applied
first, no jump is ever hidden between two observations and the recorded latent
$z_k$ is the \emph{post}-jump state; there is accordingly no event-localisation
difficulty in these two datasets.

\paragraph{Cyclic jumps.}
The decay rates are drawn per trajectory,
$a_1\sim\mathcal U(-0.12,-0.08)$ and $a_2\sim\mathcal U(-0.18,-0.12)$, so the
flow is heterogeneous across the dataset. Here $\sigma=0.10$, $b=0$.

\paragraph{Cyclic jumps, fixed decay.}
Identical in every other respect, but the rates are held at
$(a_1,a_2)=(-0.10,-0.15)$ for every trajectory. The two variants therefore form
a paired control isolating sequence-level decay heterogeneity: the random stream
is arranged so that all other sampled quantities are unchanged.

\subsection{Stochastic events: the two PDMP datasets}

Both PDMP datasets are marked Poisson jump processes with a nonlinear,
state-dependent jump map. The event \emph{rate} is homogeneous: events arrive as
a Poisson process of constant intensity $\lambda=0.3$, independent of the state.
Between events the latent flows by the exact solution of
$\dot z=a\odot z$ with $a=(-0.10,-0.15)^{\!\top}$,
\begin{equation}
z(t)\;=\;e^{a\,(t-s)}\odot z(s)\qquad\text{for }s\le t\text{ with no event in }(s,t],
\end{equation}
starting from $z(0)=0$. At an event time $\tau$ a mark is drawn and the state is
displaced by a fixed two-layer network,
\begin{equation}
U_\tau\sim\mathcal N(0,I_2),
\qquad
z(\tau) \;=\; z(\tau^-) \;+\; s_J\,W_1\tanh\!\big(W_2\,z(\tau^-)+W_3\,U_\tau+\beta\big),
\label{eq:pdmpjump}
\end{equation}
with jump scale $s_J=1$ and weights
$W_1\in\mathbb R^{2\times 8}$, $W_2,W_3\in\mathbb R^{8\times 2}$,
$\beta\in\mathbb R^{8}$, each entry drawn once as
$\mathcal N(0,0.5^{2})$ under a fixed seed and then held constant across all
trajectories and across both PDMP datasets. The jump is thus state-dependent
through $W_2z(\tau^-)$ while the intensity is not.

Events are simulated interval by interval: on $(t_{k-1},t_k]$ of length
$\Delta_k$ the count is $N_k\sim\mathrm{Poisson}(\lambda\Delta_k)$ and the event
times are the order statistics of $N_k$ uniform draws on that interval. Both
datasets use horizon $H=13$ with $52$ observations and $\sigma=0.15$.

Because the flow is pure decay from $z(0)=0$, a trajectory remains identically
zero until its first event: with $\lambda H=3.9$ expected events the latent is
exactly $(0,0)$ for the first twelve to thirteen observations on average, and a
fraction $e^{-\lambda H}\approx 2\%$ of trajectories contain no event at all and
therefore consist of $b$ plus observation noise throughout. The offset $b$ is
added to the observations only; the stored ground-truth latent is not offset, so
any latent-space metric is compared against a target displaced from the data
by $b$.

\paragraph{Regular PDMP.}
A deterministic, equally spaced grid $t_k = kH/52$, so $\Delta_k\equiv 0.25$,
with no observation offset, $b=(0,0)^{\!\top}$.

\paragraph{Irregular PDMP.}
The grid is itself random and differs for every trajectory: $52$ points are
drawn uniformly on $(0,H]$ and sorted,
\begin{equation}
t_1<\dots<t_{52},\qquad \{t_k\}=\operatorname{sort}\big(H\cdot\mathcal U(0,1)^{52}\big),
\end{equation}
resampled if not strictly increasing. The inter-observation gaps are therefore
exponential-like rather than constant, which is what makes this the irregular
grid. It additionally carries an observation offset $b=(1,1)^{\!\top}$, and it is
generated from a different seed stream, so the two PDMP datasets are
\emph{not} paired trajectory by trajectory. A comparison between them therefore
confounds the irregular grid with a constant shift and with independent sampling
noise, and is not a clean grid-only ablation.

\begin{table}[t]
\centering\small
\caption{The six synthetic datasets. All have latent and observation dimension
$2$ and $52$ observations per trajectory, of which $44$ are context}
\label{tab:datasets}
\begin{tabular}{@{}lllccl@{}}
\toprule
Dataset & Event times & Jump / innovation & $\sigma$ & Grid & Per-traj.\ randomness \\
\midrule
Deterministic rotation & none & none & $0.10$ & regular & $\theta$, $z_0$ \\
Damped rotation & none & Gaussian, $\varsigma{=}0.02$ & $0.10$ & regular & $\alpha,\omega$, $z_0$ \\
Cyclic jumps & every $5$ steps & $4$-cycle, fixed dirs. & $0.10$ & regular & $a_1,a_2$, $c_1,c_2$, $z_0$ \\
Cyclic jumps, fixed decay & every $5$ steps & $4$-cycle, fixed dirs. & $0.10$ & regular & $c_1,c_2$, $z_0$ \\
Regular PDMP & Poisson, $\lambda{=}0.3$ & $W_1\tanh(\cdot)$, Eq.~\eqref{eq:pdmpjump} & $0.15$ & regular & events, marks \\
Irregular PDMP & Poisson, $\lambda{=}0.3$ & $W_1\tanh(\cdot)$, Eq.~\eqref{eq:pdmpjump} & $0.15$ & \emph{random} & events, marks, grid \\
\bottomrule
\end{tabular}
\end{table}
\section{Evaluation definitions and protocols}
\label{app:metrics}
Let $\mathcal I$ contain the held-out scalar coordinates that are actually observed. For predictive latent paths $Z^{(1:S)}$, define the point prediction
\begin{equation}
\widehat y_{t,c}=\frac1S\sum_{s=1}^S\mu_\phi(Z_t^{(s)})_c.
\end{equation}
MSE is $|\mathcal I|^{-1}\sum_{(t,c)\in\mathcal I}(y_{t,c}-\widehat y_{t,c})^2$, and RMSE is its square root. For observation samples $\widetilde y^{(s)}$, the interval endpoints are the empirical $0.05$ and $0.95$ quantiles at each coordinate. Coverage and mean width are
\begin{equation}
\widehat C_{.90}=\frac1{|\mathcal I|}\sum_{(t,c)\in\mathcal I}\ind\{q_{.05,t,c}\leq y_{t,c}\leq q_{.95,t,c}\},
\qquad
\widehat W_{.90}=\frac1{|\mathcal I|}\sum_{(t,c)\in\mathcal I}(q_{.95,t,c}-q_{.05,t,c}).
\end{equation}

\paragraph{Joint and marginal predictive density.}
For a single trajectory with $D_f$ observed forecast scalars, the joint-path estimate is
\begin{equation}
\widehat{\mathrm{LPD}}_{\rm joint}
=\frac1{D_f}\left[\operatorname{logsumexp}_{s=1}^S
\left\{\sum_{(t,c)\in\mathcal I}\log p_\phi(y_{t,c}\mid Z_t^{(s)})\right\}-\log S\right].
\label{eq:jointlpd}
\end{equation}
The marginal estimate instead averages separate scalar mixture log densities,
\begin{equation}
\widehat{\mathrm{LPD}}_{\rm marginal}
=\frac1{D_f}\sum_{(t,c)\in\mathcal I}
\left[\operatorname{logsumexp}_{s=1}^S\log p_\phi(y_{t,c}\mid Z_t^{(s)})-\log S\right].
\end{equation}

Synthetic sequences have the same forecast dimensions, so the reported per-scalar joint score can be averaged across trajectories. These estimates assess different features of a distribution: the joint version preserves a shared latent path across future times, whereas the marginal version evaluates each scalar separately. Both approximate mixture densities with a finite sample ensemble.

For observation samples $\widetilde y^{(1:S)}$, the continuous ranked
probability score at one scalar coordinate is estimated from the sample
predictive distribution by the fair (unbiased) estimator
\begin{equation}
\widehat{\mathrm{CRPS}}(\widetilde y^{(1:S)}, y)
=\frac1S\sum_{s=1}^S\bigl|\widetilde y^{(s)}-y\bigr|
-\frac1{2S(S-1)}\sum_{s=1}^S\sum_{s'=1}^S\bigl|\widetilde y^{(s)}-\widetilde y^{(s')}\bigr|,
\end{equation}
which is computed in $O(S\log S)$ from the order statistics
$\widetilde y^{(1)}\le\dots\le\widetilde y^{(S)}$ as
$\sum_s(2s-S-1)\,\widetilde y^{(s)}/\bigl(S(S-1)\bigr)$ for the second term. The reported CRPS is the
mean over the held-out coordinates,
\begin{equation}
\mathrm{CRPS}=\frac1{|\mathcal I|}\sum_{(t,c)\in\mathcal I}
\widehat{\mathrm{CRPS}}\bigl(\widetilde y^{(1:S)}_{t,c},\,y_{t,c}\bigr).
\end{equation}

This is a marginal score: it evaluates each time and coordinate separately and
is blind to dependence across the horizon, so a model that gets every marginal
right while sampling incoherent paths is not penalised. Horizon CRPS scores the
aggregate instead. Summing over the observed times of each held-out window
before scoring,
\begin{equation}
Y_c=\sum_{t:(t,c)\in\mathcal I}y_{t,c},
\qquad
\widetilde Y^{(s)}_c=\sum_{t:(t,c)\in\mathcal I}\widetilde y^{(s)}_{t,c},
\qquad
\mathrm{CRPS}_{\Sigma}=\frac1{|\mathcal C|}\sum_{c\in\mathcal C}
\widehat{\mathrm{CRPS}}\bigl(\widetilde Y^{(1:S)}_c,\,Y_c\bigr),
\end{equation}
where $\mathcal C$ indexes the window--coordinate pairs with at least one
observation. Because the sum is taken within a path before the score is
applied, $\mathrm{CRPS}_{\Sigma}$ depends on the joint law of the predicted
trajectory and not only on its marginals. On REFIT the observable is power, so
this quantity is reported in energy units as
$E=\sum_t P_t/(60\times1000)$~kWh and named energy CRPS.

\begin{figure}
    \centering
    \includegraphics[width=\linewidth]{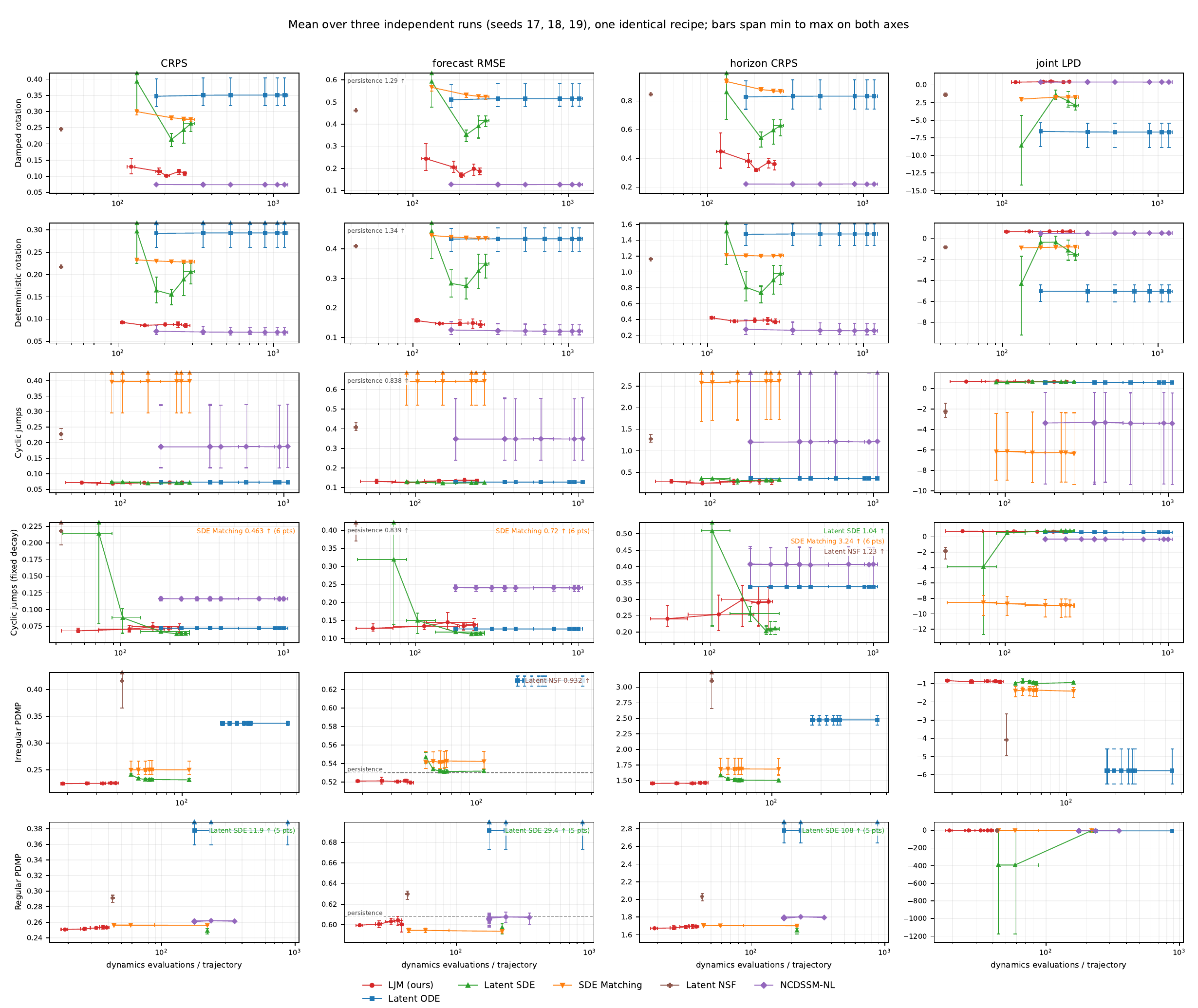}
    \caption{Forecast quality against latent update cost on the six synthetic systems. Rows are systems, columns are CRPS, forecast RMSE, horizon CRPS, joint LPD of the central 90\% predictive interval, and joint log predictive density per scalar. The horizontal axis is the number of dynamics-network evaluations per trajectory over the context window, on a logarithmic scale. LJM is shown at each of $\beta_{\text{sparse}} \in \{0, 0.03, 0.1, 0.3, 1\}$; points of one family are joined in order of cost.}
    \label{fig:synthetic-pareto}
\end{figure}

\subsection{Real-data architectures and training protocols}
\label{app:real}
On REFIT every family is built around the same components, so the comparison
varies the latent dynamics and holds everything else fixed. All three share one
history encoder and one decoder, and their totals land within $1.5\%$ of each
other at roughly $1.9$M parameters --- what differs is only how each represents
the latent path between observations. LJM spends its budget on the jump map,
the prior mark and intensity networks and the causal recognition heads; SDE
Matching on a marginal posterior, a prior drift and a diffusion network; the
Neural Stochastic Flow on a learned transition law, a GRU posterior and a
bridge. Aggregate household power is positive and heavy-tailed, so every family
emits through a log-normal decoder on the household-scaled series
$x = (P + 1\,\mathrm{W})/s_h$, with $s_h$ the $0.99$ quantile over that
household's training minutes. LJM uses a latent dimension of $32$ in four
blocks of eight with eight-dimensional Gaussian marks. The USHCN models are smaller, as that dataset is: LJM there has $17{,}596$ parameters at latent
dimension $8$ in two blocks of four with two-dimensional marks.

Training is matched within each dataset and differs between them only because
the datasets do. On REFIT every family takes $30{,}000$ AdamW steps at batch
$32$ with weight decay $10^{-5}$, gradient clipping at $1.0$ and a
warmup-then-cosine schedule; the peak learning rate is $6\times10^{-4}$ for LJM
and $2\times10^{-4}$ for both baselines, and the KL weight is $1.0$ for LJM and
SDE Matching and $0.1$ for the flow, warmed in over the first $2{,}000$ steps.
The flow additionally trains on random contiguous sub-sequences of length $50$.
On USHCN every family runs $450$ epochs at batch $32$ with the same optimiser,
weight decay and clipping, a peak learning rate of $4\times10^{-4}$ decayed to
$2.5\times10^{-5}$, a warmup of eight epochs and a KL weight of $5.0$; the
decoder mean is learned with the observation scale held fixed.

\section{How sparsity changes the role of future events} \label{app:future_events}

\subsection{Task structure and intervention}
The cyclic-jump generator uses a fixed event schedule shared across
trajectories: a jump occurs every $m=5$ grid steps and the
mark sequence repeats $(+c_1,+c_2,-c_1,-c_2)$. Trajectory-level randomness
in the base task comes from the two decay rates, magnitudes $c_1,c_2$,
initial state, and observation noise. The fixed-decay variant retains the
scheduled event times and repeating mark-sign sequence. In contrast,
event timing is random in the PDMP tasks. The two rotation tasks evaluate
smooth oscillatory trajectories.

For each fitted synthetic configuration, we compare the learned prior
rollout (Full) with a rollout whose prior intensity is held at its
positive floor (Floor). Posterior context inference and the fitted model
parameters remain unchanged. The five tested values are
$\beta_{\mathrm{sparse}}\in\{0,0.03,0.1,0.3,1\}$. Each condition uses
1,000 validation trajectories and 128 predictive paths. Floor is a
near-zero-event intervention, not a zero-rate process; over the cyclic
eight-unit forecast window, its expected count is $0.008$ events per
path. Table~\ref{tab:future_event_sweep} reports the complete sweep.

\begin{table}[p]
\centering
\small
\setlength{\tabcolsep}{8pt}
\renewcommand{\arraystretch}{1.04}
\caption{Future-event suppression across all six synthetic systems and
five event-count penalties. Full uses the learned prior intensity;
Floor holds it at its positive minimum during forecasting. Context
inference and model parameters are unchanged. Each condition uses 1,000
validation trajectories and 128 predictive paths. The last column is the
mean accepted future event count under Full. Coverage is rounded to two decimal places.}
\label{tab:future_event_sweep}
\begin{tabular}{@{}r rr rr r@{}}
\toprule
& \multicolumn{2}{c}{RMSE} & \multicolumn{2}{c}{90\% coverage}
& \shortstack{Future events/path} \\
\cmidrule(lr){2-3}\cmidrule(lr){4-5}
$\beta_{\mathrm{sparse}}$ & Full & Floor & Full & Floor & Full \\
\midrule
\multicolumn{6}{@{}l}{\textit{Cyclic jumps}} \\
$0$ & $0.124$ & $1.312$ & $0.88$ & $0.37$ & $50.6$ \\
$0.03$ & $0.126$ & $1.225$ & $0.88$ & $0.38$ & $35.2$ \\
$0.1$ & $0.126$ & $0.808$ & $0.88$ & $0.39$ & $22.7$ \\
$0.3$ & $0.127$ & $0.473$ & $0.88$ & $0.56$ & $14.7$ \\
$1$ & $0.121$ & $0.120$ & $0.86$ & $0.86$ & $0.4$ \\
\addlinespace[3pt]
\multicolumn{6}{@{}l}{\textit{Cyclic jumps, fixed decay}} \\
$0$ & $0.121$ & $1.499$ & $0.88$ & $0.34$ & $52.0$ \\
$0.03$ & $0.123$ & $1.328$ & $0.88$ & $0.37$ & $34.3$ \\
$0.1$ & $0.125$ & $1.176$ & $0.89$ & $0.37$ & $21.4$ \\
$0.3$ & $0.130$ & $0.576$ & $0.89$ & $0.31$ & $15.2$ \\
$1$ & $0.116$ & $0.116$ & $0.87$ & $0.87$ & $0.4$ \\
\addlinespace[3pt]
\multicolumn{6}{@{}l}{\textit{Damped rotation}} \\
$0$ & $0.175$ & $0.910$ & $0.96$ & $0.38$ & $47.3$ \\
$0.03$ & $0.213$ & $0.841$ & $0.96$ & $0.41$ & $40.8$ \\
$0.1$ & $0.314$ & $0.879$ & $0.93$ & $0.39$ & $36.4$ \\
$0.3$ & $0.350$ & $0.880$ & $0.94$ & $0.37$ & $27.5$ \\
$1$ & $0.274$ & $0.778$ & $0.95$ & $0.47$ & $17.3$ \\
\addlinespace[3pt]
\multicolumn{6}{@{}l}{\textit{Deterministic rotation}} \\
$0$ & $0.147$ & $0.712$ & $0.97$ & $0.46$ & $44.3$ \\
$0.03$ & $0.147$ & $0.779$ & $0.97$ & $0.44$ & $43.6$ \\
$0.1$ & $0.190$ & $0.822$ & $0.96$ & $0.40$ & $35.4$ \\
$0.3$ & $0.163$ & $0.784$ & $0.97$ & $0.50$ & $26.1$ \\
$1$ & $0.184$ & $0.613$ & $0.97$ & $0.54$ & $15.4$ \\
\addlinespace[3pt]
\multicolumn{6}{@{}l}{\textit{Irregular PDMP}} \\
$0$ & $0.519$ & $0.579$ & $0.85$ & $0.70$ & $8.7$ \\
$0.03$ & $0.521$ & $0.581$ & $0.85$ & $0.70$ & $8.1$ \\
$0.1$ & $0.518$ & $0.573$ & $0.85$ & $0.71$ & $7.4$ \\
$0.3$ & $0.535$ & $0.586$ & $0.84$ & $0.70$ & $5.5$ \\
$1$ & $0.520$ & $0.572$ & $0.86$ & $0.72$ & $4.2$ \\
\addlinespace[3pt]
\multicolumn{6}{@{}l}{\textit{Regular PDMP}} \\
$0$ & $0.605$ & $0.663$ & $0.83$ & $0.69$ & $9.1$ \\
$0.03$ & $0.607$ & $0.661$ & $0.83$ & $0.69$ & $8.5$ \\
$0.1$ & $0.605$ & $0.668$ & $0.83$ & $0.69$ & $7.3$ \\
$0.3$ & $0.599$ & $0.661$ & $0.84$ & $0.70$ & $6.0$ \\
$1$ & $0.605$ & $0.664$ & $0.85$ & $0.70$ & $4.3$ \\
\bottomrule
\end{tabular}
\end{table}

\subsection{Three observed regimes}
\paragraph{Scheduled cyclic jumps.}
With the learned event mechanism, forecast RMSE varies little across the
penalty sweep: $0.121$--$0.127$ for cyclic jumps and $0.116$--$0.130$
with fixed decay. Dependence on future events nevertheless changes
substantially. Without the event-count penalty, the two fits generate
$50.6$ and $52.0$ future events per path, and suppression raises RMSE
to $1.312$ and $1.499$. At the intermediate penalties, suppressing
events remains damaging. Only the largest tested penalty produces
forecasts that are nearly unchanged by suppression: RMSE is
$0.121/0.120$ for Full/Floor on cyclic jumps and $0.116/0.116$ with
fixed decay, at approximately $0.4$ future events per path under Full.

The path diagnostics agree with this distinction. For the sparse cyclic
fit, the frequency of a largest one-step output increment above $1.0$
is $0.84$ under Full and $0.88$ under Floor, compared with $0.88$ in
the reference trajectories. For the fit with no event-count penalty,
the frequency drops from $0.79$ to $0.37$ under suppression, and
per-step coverage reaches zero from the reported step~6 onward. These
results are consistent with the sparse fit representing the predictable
cycle through the inferred context state, closed-form flow, and decoder,
whereas the other fitted configurations rely on future events. The
intervention tests that reliance without retraining.

\paragraph{Randomly timed PDMPs.}
Increasing the penalty from zero to one reduces the mean future count
from $8.7$ to $4.2$ on the irregular PDMP and from $9.1$ to $4.3$ on
the regular PDMP. Full-rollout RMSE remains similar: $0.519$ versus
$0.520$ and $0.605$ versus $0.605$, respectively. Events continue to
matter throughout the sweep. Their suppression reduces coverage by
approximately $13$--$15$ percentage points at every tested penalty,
while RMSE increases to $0.572$--$0.586$ on irregular PDMP and
$0.661$--$0.668$ on regular PDMP.

The distributional diagnostics show why point error alone understates
this dependence. Across the reported PDMP diagnostics, the frequency of
a largest one-step output increment exceeding $1.0$ is approximately
$0.23$--$0.26$ in the reference trajectories, $0.12$--$0.18$ under
Full, and at most $0.02$ under Floor. In the irregular-PDMP fit at
$\beta_{\mathrm{sparse}}=0$, suppression reduces mean interval width
from $1.139$ to $0.573$ and joint-path LPD from $-0.90$ to $-4.70$.
Interval width grows with horizon from approximately $0.76$ to $1.44$
under Full, but changes from $0.60$ to $0.55$ under Floor. Future events
supply much of the forecast tail and its horizon-dependent spread. The
reported regular-PDMP tail frequency increases from $0.128$ to $0.176$
from the zero to the largest penalty, so fewer events need not imply a
weaker tail. Even the full models still underproduce the largest changes
relative to the reference trajectories.

\paragraph{Smooth rotations.}
Future-event suppression is damaging at every tested penalty, and the
penalty also affects the full-rollout accuracy. On damped rotation,
RMSE increases from $0.175$ at $47.3$ future events to $0.274$ at
$17.3$ events for the largest penalty, reaching $0.350$ at
$\beta_{\mathrm{sparse}}=0.3$. The deterministic-rotation increase is
smaller, from $0.147$ to $0.184$ between the endpoint penalties.
Under suppression, their RMSEs are $0.778$ and $0.613$, respectively,
at the largest penalty. In these fitted models, events therefore support
the learned oscillatory forecasts even though the target trajectories are
smooth. This contrasts with the sparse cyclic fits, which remain accurate
with nearly all future events suppressed.

\subsection{Additional REFIT results}
We vary the sparsity penalty of LJM on context length $512$ and show the results in Table \ref{tab:refit_beta}.

 \begin{table}[t]
\centering\small
\setlength{\tabcolsep}{4pt}
\caption{Event-count penalty sweep for the latent jump model on REFIT, 18 households, context 512, test split.}
\label{tab:refit_beta}
\begin{tabular}{@{}lrrrrrr@{}}
\toprule
$\beta_{\mathrm{sparse}}$ & $N_{\mathrm{ev}}\downarrow$ & CRPS$/s_h\downarrow$ & RMSE (W)$\downarrow$ & Cov.$\to$0.90 & Width (W) & Energy CRPS$\downarrow$\\
\midrule
0 & 205.4 & \textbf{0.0491} & 550.3 & 0.923 & 830 & \textbf{0.0627}\\
0.01 & 200.6 & 0.0496 & \textbf{549.8} & 0.914 & 858 & 0.0633\\
0.03 & 176.8 & 0.0508 & 563.8 & 0.929 & 864 & 0.0664\\
0.1 & 185.0 & 0.0507 & 559.8 & 0.928 & 876 & 0.0649\\
0.3 & 147.5 & 0.0517 & 563.5 & 0.919 & 855 & 0.0668\\
1 & 98.4 & 0.0516 & 565.4 & 0.926 & 920 & 0.0666\\
2 & 64.2 & 0.0526 & 568.6 & \textbf{0.909} & \textbf{735} & 0.0690\\
4 & 46.6 & 0.0540 & 573.9 & 0.922 & 920 & 0.0722\\
\bottomrule
\end{tabular}
\end{table}
\paragraph{Interpretation.}
Table~\ref{tab:refit_beta} varies the event penalty at context 512, and the
comparison that matters is against Table~\ref{tab:refit_contexts}. Raising
$\beta_{\mathrm{sparse}}$ from $0$ to $4$ cuts the mean accepted event count
from $205.4$ to $46.6$, a factor of $4.4$, and costs $10\%$ in normalised CRPS
($0.0491\to0.0540$). That degradation is small relative to the gap to the
baselines. At $\beta_{\mathrm{sparse}}=2$ the model uses $64.2$ events and
scores $0.0526$, still ahead of SDE Matching's $0.0544\pm0.0011$ by more than
its seed spread and well ahead of the Neural Stochastic Flow's
$0.0588\pm0.0018$; only at $\beta_{\mathrm{sparse}}=4$ does it fall level with
SDE Matching, and it beats the flow at every penalty in the sweep. The event
budget is therefore not what buys the margin: three quarters of the events can
be removed before the advantage over the stronger baseline disappears, and the
advantage over the weaker one survives the whole sweep.

Coverage stays
between $0.909$ and $0.929$ throughout and is closest to nominal at
$\beta_{\mathrm{sparse}}=2$, where the interval width is also narrowest at
$735$~W, so the sparser models are not achieving their coverage by widening.

\section{Scan versus chronological state construction}
\label{app:scan_benchmark}

\subsection{Benchmark setup}
We evaluate the same fitted models with identical inputs using two
execution strategies: the chronological flow--gated-jump loop from the
test suite and the vendored pure-PyTorch associative scan used during
training. The scan implementation uses no fused kernel. Timings are
medians measured in float32 on one NVIDIA A40. We vary the batch size $B$
and processed proposal length $L$; the latter includes padding at the
training lengths and is distinct from the accepted event count.
Table~\ref{tab:scan_benchmark} combines forward-only time,
forward--backward time, and peak memory for all eight settings. This
experiment measures posterior state construction and its backward pass,
separately from the encoder, full variational objective, and optimizer.

\begin{table}[t]
\centering
\small
\setlength{\tabcolsep}{4pt}
\renewcommand{\arraystretch}{1.14}
\caption{Matched-model posterior state construction on one NVIDIA A40
(float32; median times). Seq. is the chronological flow--gated-jump loop;
Scan is the pure-PyTorch scan used in training, without a fused kernel.
$B$ is batch size and $L$ is the processed proposal length.
$\dagger$ marks training padding lengths. Speed-up refers to the
forward--backward pass and are computed from unrounded timing measurements; timing entries are
rounded. Peak-memory units are shown in each entry.}
\label{tab:scan_benchmark}
\label{tab:scan_backward}
\label{tab:scan_memory}
\begin{tabular}{@{}lrr rr rrr rr@{}}
\toprule
& & & \multicolumn{2}{c}{Forward (ms)}
& \multicolumn{2}{c}{Forward + backward (ms)} &
& \multicolumn{2}{c}{Peak memory} \\
\cmidrule(lr){4-5}\cmidrule(lr){6-7}\cmidrule(lr){9-10}
Model & $B$ & $L$ & Seq. & Scan & Seq. & Scan & Speed-up & Seq. & Scan \\
\midrule

Cyclic & 16  & $1024^{\dagger}$ & 985   & 5.0  & 3,025  & 24 & $125\times$ & 155 MB  & 174 MB \\
Cyclic & 128 & 1024             & 1,019 & 10.9 & 3,053  & 32 & $96\times$  & 1.20 GB & 1.30 GB \\
Cyclic & 128 & 4096             & 4,120 & 41.8 & 12,506 & 93 & $134\times$ & 4.81 GB & 5.20 GB \\
\bottomrule
\end{tabular}
\end{table}

\subsection{Runtime scaling}

The larger-batch cyclic results expose the scan's increasing arithmetic
work: at $B=128$, increasing $L$ from $1024$ to $4096$ increases
forward--backward time from $32$ to $93$ ms, compared with $3{,}053$ to
$12{,}506$ ms for the chronological loop. The approximately flat timings
in smaller cases describe the measured hardware regime; the state
construction still has $O(JL)$ work and $O(J\log(1+L))$ temporal depth
for fixed block sizes and coefficient-network architectures.



\end{document}